\documentclass[letterpaper]{article}
\usepackage{enumerate}
\usepackage{amsmath}
\usepackage{amsfonts}
\usepackage{amssymb}
\usepackage{amsbsy}
\usepackage{amsthm}
\usepackage{dsfont}
\usepackage{algorithm}
\usepackage{algorithmic}
\usepackage{mathrsfs}
\usepackage{paralist}
\usepackage{epsfig}
\usepackage{subfigure}
\usepackage{makeidx} 
\usepackage{color}
\usepackage{pstricks}
\usepackage{pst-node}
\usepackage{multido} 
\usepackage{varioref}
\usepackage{wrapfig} %% for wrapping text around figures
\usepackage[sort,compress,round]{natbib}

\numberwithin{equation}{section} 

\theoremstyle{plain} \newtheorem{remark}{Remark}[section]
\theoremstyle{plain} \newtheorem{definition}{Definition}[section]
\theoremstyle{plain} 
\theoremstyle{plain} \newtheorem{assumption}{Assumption}[section]
\theoremstyle{plain} 
\theoremstyle{plain} \newtheorem{theorem}{Theorem}[section]
\theoremstyle{plain} 
\theoremstyle{plain} 
\theoremstyle{plain} \newtheorem{corollary}{Corollary}[section]

\newcommand \E {\mathop{\mbox{\ensuremath{\mathbb{E}}}}\nolimits}

\renewcommand \Pr {\mathop{\mbox{\ensuremath{\mathbb{P}}}}\nolimits}

\newcommand{\set}[1]{\left\{\, #1 \,\right\} }
\newcommand{\cset}[2]{\left\{\, #1 ~\middle|~ #2 \,\right\} }

\newcommand\Reals {{\mathds{R}}}

\newcommand \CA {{\mathcal{A}}}

\newcommand \CH {{\mathcal{H}}}

\newcommand \CM {{\mathcal{M}}}

\newcommand \CW {{\mathcal{W}}}
\newcommand \CX {{\mathcal{X}}}

\newcommand \hM {\widehat{M}}

\newcommand \defn {\mathrel{\triangleq}}
\newcommand \argmax{\mathop{\rm arg\,max}}

\newcommand \norm[1]{\left\|#1\right\|}

\DeclareMathAlphabet{\mathpzc}{OT1}{pzc}{m}{it}

\newcommand \Uniform{\mathop{\mathpzc{Unif}}\nolimits}

\newcommand \bel {\xi}

\newcommand \pol {\pi}

\newcommand \mdp {\mu}
\newcommand \MDPs {\CM}

\newcommand \Pmp {\Pr_\mdp^\pol}

\newcommand\ind[1]{\mathop{\mbox{\ensuremath{\mathbb{I}}}}\left\{#1\right\}}

\newcommand\dd{\,\mathrm{d}}

\newcommand \seq[2]{#1^{#2}}
\newcommand \pseq[3]{#1_{#2}^{#3}}
\newcommand \sam[2]{#1^{(#2)}}

\newcommand \discount {\gamma}
\newcommand \threshold {\varepsilon}
\newcommand \ntrain {\ensuremath{N_{\mathrm{dat}}}}
\newcommand \nsamples {\ensuremath{N_{\mathrm{sam}}}}
\newcommand \nrollouts {\ensuremath{N_{\mathrm{rol}}}}
\newcommand \ntrajectories {\ensuremath{N_{\mathrm{trj}}}}

\newcommand \hist {h}

\newcommand \Urange {U_{\max}}

\newcommand \dist[2]{D\left(#1 ~\middle\|~ #2\right)}
\newcommand \marg {\phi}
\newcommand \Ae {A_\epsilon^\hist}

\usepackage[nohyperref,accepted]{icml2013}
\usepackage{gnuplot-lua-tikz}

\usepackage{graphicx} % more modern
\usepackage{subfigure} 

\usepackage{natbib}

\usepackage{algorithm}
\usepackage{algorithmic}

\usepackage{icml2013} 
\icmltitlerunning{ABC Reinforcement Learning}

\begin{document} 

\twocolumn[
\icmltitle{ABC Reinforcement Learning}

% It is OKAY to include author information, even for blind
% submissions: the style file will automatically remove it for you
% unless you've provided the [accepted] option to the icml2013
% package.
\icmlauthor{Christos Dimitrakakis}{christos.dimitrakakis@gmail.com}
\icmladdress{EPFL, Lausanne, Switzerland}

\icmlauthor{Nikolaos Tziortziotis}{ntziorzi@gmail.com}
\icmladdress{University of Ioannina, Greece}

\vskip 0.3in
]

\begin{abstract} 
%  This paper introduces a simple, general framework for \emph{likelihood-free} Bayesian reinforcement learning, through Approximate Bayesian Computation (ABC). The main advantage is that we only require a prior distribution on a class of simulators (generative models). This is useful in domains where an analytical probabilistic model of the underlying process is too complex to formulate, but where detailed simulation models are available. ABC-RL allows  the use of any Bayesian reinforcement learning technique, even in this case. In addition, it can be seen as an extension of rollout algorithms to the case where we do not know what the correct model to draw rollouts from is. 
  We introduce a simple, general framework for \emph{likelihood-free} Bayesian reinforcement learning, through Approximate Bayesian Computation (ABC). The advantage is that we only require a prior distribution on a class of simulators. This is useful when a probabilistic model of the underlying process is too complex to formulate, but where detailed simulation models are available. ABC-RL allows  the use of any Bayesian reinforcement learning technique in this case. It can be seen as an extension of simulation methods to both planning and inference.
We experimentally demonstrate the potential of this approach in a comparison with LSPI. Finally, we introduce a theorem showing that ABC is sound.
\end{abstract} 

\section{Introduction}
\label{sec:introduction}
Bayesian reinforcement learning~\citep{strens2000bayesian,vlassis2012bayesian} is the decision-theoretic approach~\citep{Degroot:OptimalStatisticalDecisions}
to solving the reinforcement learning problem. However, apart from the
fact that calculating posterior distributions and the Bayes-optimal
decision is frequently intractable~\citep{NIPS2007:ross:bapomdp,duff2002olc}, another major difficulty is the specification of the prior and model
class. While there exist a number of non-parametric Bayesian model
classes which can be brought to bear for estimation of the dynamics of
an unknown process, it may not be a trivial matter to select the
correct class and prior. On the other hand, it is frequently known that the process can be approximated well by a complex parametrised simulator. The question is how to take advantage of this knowledge when the best simulator parameters are not known.

We propose a simple, general, reinforcement learning framework employing the principles of Approximate Bayesian Computation (ABC, see~\citep{csillery2010ABC} for an overview) for performing Bayesian inference using simulation. In doing so, we extend rollout algorithms for reinforcement learning, such as those described in~\citep{BertsekasTsitsiklis:NDP,Bertsekas:Rollout:2005,dimitrakakis+lagoudakis:mlj2008,lagoudakis2003rcpi}, to the case where we do not know what the correct model to draw rollouts from is.

We show how to use ABC to compute approximate posteriors over a set of environment models in the context of reinforcement learning. This includes a simple but general theoretical result on the quality of ABC posterior approximations. Finally, building on previous approaches to Bayesian reinforcement learning, we propose a strategy for selecting policies in this setting.

\subsection{The setting}
\label{sec:setting}
In the \emph{reinforcement learning problem}, an agent is acting in some unknown environment $\mdp$, according to some policy $\pol$. The agent's policy is a procedure for selecting a sequence of actions, with the action at time $t$ being $a_t \in \CA$. The environment reacts to this sequence with a corresponding sequence of observations $x_t \in \CX$ and rewards $r_t \in \Reals$. This interaction may depend on the complete history\footnote{A history may include multiple trajectories in episodic environments.} $h \in \CH$, where $\CH \defn (\CX \times \CA \times \Reals)^*$ is the set of all state action reward sequences, as neither the agent or the environment are necessarily finite-order Markov. For example, the agent may learn, or the environment may be partially observable. 

In this paper, we use a number of shorthands to simplify notation. Firstly,  we denote the (random) probability measure for the agent's action at time $t$ by:
\begin{equation}
  \label{eq:policy}
  \pol_t(A) \defn \Pr^\pol(a_t \in A \mid \seq{x}{t}, \seq{r}{t}, \seq{a}{t-1}),
\end{equation}
where $\seq{x}{t}$ is a shorthand for the sequence $(x_i)_{i=1}^t$;
similarly, we use $\pseq{x}{k}{t}$ for $(x_i)_{i=k}^t$. We denote the environment's response at time $t+1$ given the history at time $t$ by:
\begin{equation}
  \label{eq:mdp}
  \mdp_t(B) \defn \Pr_\mdp((x_{t+1}, r_{t+1}) \in B \mid \seq{x}{t}, \seq{r}{t}, \seq{a}{t}).
\end{equation}
In a further simplification, we shall also use $\pol_t(a_t)$ for the probability (or density) of the action actually taken by the policy at time $t$, and similarly, $\mdp_t(x_t)$ for the realised observation. Finally, we use $\Pmp$ to denote joint distributions on action, observation and reward sequences under the environment $\mdp$ and policy $\pol$.
% and $h_t =  \langle \seq{x}{t}, \seq{r}{t}, \seq{a}{t} \rangle$ to denote the history until time $t$.

The agent's goal is determined through its utility:
\begin{equation}
  \label{eq:utility}
  U \defn \sum_{t=1}^\infty \gamma^{t-1} r_t,
\end{equation}
which is a discounted sum of the total instantaneous rewards obtained, with $\gamma \in [0,1]$. Without loss of generality, we assume that $U \in [0, \Urange]$.  The optimal policy  maximises the expected utility $\E_{\mdp}^{\pol} U$.
As in the reinforcement learning problem the environment $\mdp$ is \emph{unknown}, this maximisation is ill-posed. Intuitively, we can increase the expected utility by either: 
\begin{inparaenum}[(i)]
\item Trying to better estimate $\mdp$ in order to perform the maximisation later (exploration), or
\item Use a best-guess estimate of $\mdp$ to obtain high rewards (exploitation).
\end{inparaenum}

In order to solve this trade-off, we can adopt a Bayesian viewpoint~\citep{Degroot:OptimalStatisticalDecisions,savage1972fs}, where we consider a (potentially infinite) set of environment models $\MDPs$. In particular, we select a \emph{prior probability} measure $\bel$ on $\MDPs$. For an appropriate subset $B \subset \MDPs$, the quantity $\bel(B)$ describes our initial belief that the correct model lies in $B$. We can now formulate the alternative goal of maximising the expected utility with respect to our prior:
\begin{equation}
  \label{eq:bayes-utility}
  \E_{\bel}^{\pol} U = \int_\MDPs (\E_{\mdp}^{\pol} U) \dd{\bel(\mdp)}.
\end{equation}
We can now formalise the problem as finding a policy $\pol_\bel^* \in \argmax_\pol \E_{\bel}^{\pol} U$. Any such policy is \emph{Bayes-optimal}, as it solves the exploration-exploitation problem with respect to our prior belief.

\subsection{Related work and our contribution}
\label{sec:related-work}
The first difficulty when adopting a Bayesian approach to sequential decision making is that finding the policy maximising \eqref{eq:bayes-utility} is  hard~\citep{duff2002olc} even in restricted classes of policies~\citep{dimitrakakis:mmbi:ewrl:2011}. On the other hand, simple heuristics such as Thompson sampling~\citep{thompson1933lou,strens2000bayesian} provide an efficient trade-off~\citep{agrawal:thompson,Kaufmann:Thompson} between exploration and exploitation. Alghough other heuristics exist~\citep{Kolter-Ng:NearBayesianExploration,DBLP:conf/ijcai/CastroP07,strens2000bayesian,araya2012near,poupart2006asd}, in this paper we focus on an approximate version of Thompson sampling for reasons of simplicity. 
The second difficulty is that in many interesting problems, the exact posterior calculation may be intractable, mainly due to partial observability~\citep{Poupart:ModelBayesBAPOMDP,NIPS2007:ross:bapomdp}. Interestingly, an ABC approach would not suffer from this problem for reasons that will be made clear in the sequel.

The most fundamental difficulty in a Bayesian framework is specifying a generative model class: it is not always clear what is the best model to use for an application. 
However, frequently we have access to a class of \emph{parametrised simulators} for the problem. Therefore, one reasonable approach is to find a good policy for a simulator in the class, and then apply it to the actual problem. Methods for finding good policies using simulation have been extensively studied before~\citep{BertsekasTsitsiklis:NDP,Bertsekas:Rollout:2005,Wu-Zilberstein-RolloutSampling-DecPOMDP:UAI-2010,gabillon:icml2011,dimitrakakis+lagoudakis:mlj2008}. However, in all those cases simulation was performed on a simulator with \emph{fixed} parameters. 

Approximate Bayesian Computation (ABC)~\citep[see][for an overview]{csillery2010ABC,marin2011ABC} is a general framework for likelihood-free Bayesian inference via simulation. It has been developed because of the existence of applications, such as econometric modelling~\citep[e.g.][]{geweke1999using}, where detailed simulators were available, but no useful analytical probabilistic models. While ABC methods have also been used for inference in dynamical systems~\citep[e.g][]{toni2009ABC}, they have not yet been applied to the reinforcement learning problem.

This paper proposes to perform Bayesian reinforcement learning through ABC on an arbitrary class of parametrised simulators. As ABC has been widely used in applications characterised by large amounts of data and complex simulations with many unknown parameters, it may also scale well in reinforcement learning applications. The proposed methodology is generally applicable to arbitrary  problems, including partially observable environments, continuous state spaces, and stochastic Markov games.

ABC Reinforcement Learning generalises methods previously developed for simulation-based approximation of optimal policies to the Bayesian case.
While in the standard framework covered by~\citet{Bertsekas:NLP}, a particular simulator of the environment is assumed to exist, via ABC we can relax this assumption. We only need a class of parametrised simulators that contain one close to the real environment dynamics. Thus, the only remaining difficulty is computational complexity.

Finally, we provide a simple but general bound for ABC posterior computation.
This bounds the KL divergence of the approximate posterior computed via ABC and the complete posterior distribution. As far as we know, this is a new and widely applicable result, although some other theoretical results using similar assumptions appear in~\citep{jasra2010filtering} and in~\citep{dean2011asymptotic} for hidden Markov models.

Section~\ref{sec:abc} introduces ABC inference for reinforcement learning, discusses its difference from standard Bayesian inference, and presents a theorem on the quality of the ABC approximation. Section~\ref{sec:abc-rl} describes the ABC-RL framework and the ABC-LSPI algorithm for continuous state spaces. An experimental illustration is given in Sec.~\ref{sec:experiments}, followed by a discussion in Sec.~\ref{sec:conclusion}.
The appendix contains the collected proofs.

\section{Approximate Bayesian Computation}
\label{sec:abc}
Approximate Bayesian Computation encompasses a
number of likelihood-free techniques where only an approximate
posterior is calculated via simulation. We first discuss how standard Bayesian inference in reinforcement learning differs from ABC inference. We then introduce a theorem on the quality of the ABC approximation.

\subsection{Bayesian inference for reinforcement learning}

Imagine that the history $h \in \CH$ has been generated from a process $\mdp \in \MDPs$ controlled with a history-dependent policy $\pol$, something which we denote as $h \sim \Pmp$. Now consider a prior $\bel$ on $\MDPs$ with the property that $\bel(\cdot \mid \pol) = \bel(\cdot)$, i.e. that the prior is independent of the policy used. Then the posterior probability, given a history $\hist$ generated by a policy $\pol$, that $\mdp \in B$ can be written as:
\footnote{For finite $\CM$, the posterior simplifies to $\bel(\mdp \mid h, \pol) = \Pmp (h) \bel(\mdp)/ \sum_{\mdp' \in \MDPs} \Pr_{\mdp'}^{\pol} (h) \bel(\mdp')$}
\begin{align}
  \bel(B \mid h, \pol) 
  &=
  \frac{\int_B \Pmp (h) \dd{\bel}(\mdp)}
  {\int_\MDPs \Pmp (h) \dd{\bel}(\mdp)}.
  \label{eq:posterior}
\end{align}
Fortunately, the dependence on the policy can be removed, since the posterior is the same for all policies that put non-zero mass on the observed data:
\begin{remark}
  Let $h \sim \Pmp$. Then $\forall \pol' \neq \pol$ such that $\Pr_{\mdp}^{\pol'}(h) > 0$,  $\bel(B \mid h, \pol ) = \bel(B \mid h, \pol')$.
  \label{rem:policy-dependence}
\end{remark}
Consequently, when calculating posteriors, the policy employed need not be considered, even when the process and policy depend on the complete history. In the ABC setting we do not have direct access to the probabilities $\mdp_t$, for the models $\mdp$ in our model class $\MDPs$. However, we can always generate observations from any model: $x_{t+1}  \sim \mdp_t$. This idea is used by ABC to calculate approximate posterior distributions.

\subsection{ABC inference for reinforcement learning}
\label{sec:abc-inference-rl}
The main idea of ABC is to approximate samples from the posterior distribution via simulation. We produce a sequence of sample models $\sam{\mdp}{k}$ from the prior $\bel$, and then generate data $\sam{\hist}{k}$ from each. If the generated data is ``sufficiently close'' to the history $\hist$, then the $k$-th model is accepted as a sample from the posterior $\bel(\mdp \mid \hist)$. More specifically, ABC requires that we define an approximately sufficient statistic $f : \CH \to \CW$ on some normed vector space $(\CW, \|\cdot\|)$. If $\|f(\hist) - f(\sam{\hist}{k})\| \leq \threshold$ then $\sam{\mdp}{k}$ is accepted as a sample from the posterior. Algorithm~\ref{alg:abc-post} gives the sampling method in detail for reinforcement learning. An important difference with the standard ABC posterior approximation, as well as exact inference, is the dependency on $\pol$.

Note that even though Remark~\ref{rem:policy-dependence} declares that the posterior is independent of the policy used, when using ABC this is no longer true. We must maintain the complete policy used until then to generate samples, otherwise there is no way to generate a sequence of observations.\footnote{For episodic problems, we must maintain the sequence of policies used.}
\begin{algorithm}[ht]
  \begin{algorithmic}
    \STATE \textbf{input} Prior $\bel$ on $\CM$, history  $h \in \CH$, threshold $\threshold$, statistic $f: \CH \to \CW$, policy $\pol$, maximum number of samples $\nsamples$, stopping condition $\tau$.
    \STATE $\hM = \emptyset$.
    \FOR {$k = 1, \ldots, \nsamples$}
    \STATE $\mdp^{(k)} \sim \bel$.
    \STATE $h^{(k)} \sim P_{\mdp^{(k)}}^\pol$
    \IF {$\norm{f(h) - f(h^{(k)})} < \threshold$}
    \STATE $\hM := \hM \cup \set{\mdp^{(k)}}$.
    \ENDIF
    \IF {$\tau$}
    \STATE \textbf{break}
    \ENDIF
    \ENDFOR
    \STATE \textbf{return} $\hM$
  \end{algorithmic}
  \caption{ABC-RL-Sample}
  \label{alg:abc-post}
\end{algorithm}
Intuitively, the algorithm can basically be seen as generating rollouts from a number of simulators, sampled from our prior distribution. The sampled set of simulators with a sufficient close statistic is then an approximate sample from our posterior distribution. The first question is what types of statistics we need. 

In fact, just as in standard ABC, if the statistic is sufficient, then the samples will be generated according to the posterior.
\begin{corollary}
  If $f$ is a sufficient statistic, then the set $\hM$ returned by Alg.~\ref{alg:abc-post} for $\epsilon = 0$ is a sample from the posterior.
  \label{cor:statistic}
\end{corollary}
The (standard) proof is deferred to the appendix. Thus, for $\epsilon = 0$, when the statistic is sufficient, the sampling distribution and the posterior are identical. However, things are not so clear when $\epsilon > 0$. 

We now provide a simple theorem which characterises the relation of the approximate posterior to the true posterior, when we use a (not necessarily sufficient) statistic with threshold $\epsilon > 0$.
First, we remind the definition of the KL-divergence.
\begin{definition}
  The KL-divergence $D$ between two probability measures $\bel, \bel'$ on $\MDPs$ is
  \begin{equation}
    \label{eq:divergence}
    \dist{\bel}{\bel'} \defn \int_\MDPs \ln \frac{\dd\bel(\mdp)}{\dd\bel'(\mdp)} \dd\bel(\mdp).
  \end{equation}
  \label{def:divergence}
\end{definition}
In order to prove meaningful results, we need some additional assumptions on the likelihood function. In this particular case, we simply assume that it is smooth (Lipschitz) with respect to the statistical distance:
\begin{assumption}
  For a given policy $\pol$, for any $\mdp$, and histories $x, h \in \CH$,
  there exists $L > 0$ such that $\left|\ln \left[\Pmp (h)/ \Pmp(x)\right]\right| \leq L \|f(h) - f(x)\|$.
  \label{ass:lipschitz}
\end{assumption}
We note in passing that this assumption is related to the notion of differential privacy~\citep{dwork2009differential}, from which it was inspired.

We now can state the following theorem, whose proof can be found in the appendix, which generalises the previous corollary.
\begin{theorem}
  Under a policy $\pol$ and statistic $f$ satisfying Assumption~\ref{ass:lipschitz}, the approximate posterior distribution $\bel_\epsilon(\cdot \mid \hist)$ satisfies:
  \begin{equation}
    \label{eq:posterior-divergence}
    \dist{\bel(\cdot \mid \hist)}{\bel_\epsilon(\cdot \mid \hist)} \leq \ln |\Ae| + 2 L \epsilon,
  \end{equation}
  where $\Ae \defn \cset{z \in \CH}{\|f(z) - f(\hist)\| \leq \epsilon}$ is the $\epsilon$-ball around the observed history $\hist$ with respect to the statistical distance and $|\Ae|$ denotes its size.
  \label{the:posterior-divergence}
\end{theorem}
%%For the purposes of this analysis, 
The divergence depends on the statistic in the following ways. Firstly, it approaches 0 as $\epsilon \to 0$. Secondly, it is smaller for smoother likelihoods. However, because of the dependence on the size of the $\epsilon$-ball\footnote{For discrete observations this is simply the counting measure of the ball. For more general cases it can be extended to an appropriate measure.} around the observed statistic, the statistic cannot be arbitrarily smooth. Nevertheless, it may be the case that a sufficient statistic is not required for good performance.
Since in reinforcement learning we are mainly interested in the utility rather than in system identification, we may be able to get good results by using utility-related statistics.

\paragraph{Observation-based statistics} A simple idea is to select features on which to calculate statistics. Discounted cumulative feature expectation are especially interesting, due to their connection with value functions~\citep[e.g.][Sec.~6.9.2]{Puterman:MDP:1994}. The main drawback is that this adds yet another hyper-parameter to tune.  In addition, unlike econometrics or bioinformatics, we may not be interested in model identification \emph{per se}, but only in finding a good policy.

\paragraph{Utility-based statistics} Quantities related to the utility may be a good match for reinforcement learning.  In the simplest case, it may be sufficient to only consider unconditional moments of the utility, which is the approach followed in this paper. However, these may only trivially satisfy Ass.~\ref{ass:lipschitz} for arbitrary policies. Nevertheless, as we shall see, even a very simple such statistic has a reasonably good performance.

\subsection{A Hoeffding-based utility statistic}
\label{sec:hoeffding-statistic}
 In particular,  given a history $h$  including $\ntrain$ trajectories in the environment, with the $i$-th trajectory obtaining utility $U^{(i)}$, we obtain a mean estimate $\hat{\E}^{\ntrain} U \defn \frac{1}{\ntrain} U^{(i)}$.
We then obtain a history $\hat{h}^{(k)}$ containing $\ntrajectories$ trajectories from the sampled environment $\mdp^{(k)}$ and construct the mean estimate $\hat{\E}^{\ntrajectories}_{k} U$.
In order to test whether these are close enough, we use the Hoeffding inequality~\citep{Hoeffding:SumInequalities}. In fact, it is easy to see that, with probability at least $1 - \delta$, $|\E^{\pol}_{\mdp} U - \E^{\pol}_{\mdp^{(k)}} U|$ is lower bounded by:
\begin{equation}
|\hat{\E}^\ntrain U - \hat{\E}^{\ntrajectories}_{k} U| -
\Urange \sqrt{\frac{\ln(2/\delta) (\ntrain + \ntrajectories)}{2 \ntrain \ntrajectories}},
\label{eq:statistic}
\end{equation}
where $\Urange$ is the range of the utility function.
We then use \eqref{eq:statistic} as the statistical distance $\norm{f(h) - 
f(h^{(k)})}$ between the observed history $h$ and the sampled history $h^{(k)}$.
The advantage of using this statistic is that the more data we have, it becomes harder to accept a sample.

This statistic has two parameters. Firstly, the error probability $\delta$, which does not need to be very small in practice, as the Hoeffding bound is only tight for high-variance distributions.  The second parameter is $\ntrajectories$. This does not need to be very large,
since it only makes a marginal difference in the bound when $\ntrajectories \gg \ntrain$. An illustration of the type of samples obtained with this statistic is given in Figure~\ref{fig:pendulum-error}, which shows the dependency of the approximate posterior distribution on the threshold $\epsilon$ when conditioned on a fixed amount $\ntrain$ of training trajectories.

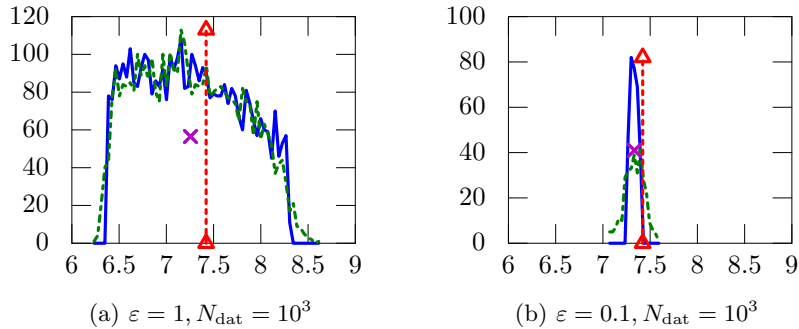
\begin{figure*}[htb]
  \centering
  \subfigure[$\threshold = 1, \ntrain = 10^3$]{
    \begin{tikzpicture}[gnuplot]
%% generated with GNUPLOT 4.6p1 (Lua 5.1; terminal rev. 99, script rev. 100)
%% Wed 09 Jan 2013 01:57:45 PM CET
\path (0.000,0.000) rectangle (5.333,4.000);
\gpcolor{color=gp lt color border}
\gpsetlinetype{gp lt border}
\gpsetlinewidth{0.50}
\draw[gp path] (1.012,0.616)--(1.263,0.616);
\draw[gp path] (4.780,0.616)--(4.529,0.616);
\gpcolor{rgb color={0.000,0.000,0.000}}
\node[gp node right,font={\fontsize{10pt}{12pt}\selectfont}] at (0.828,0.616) {0};
\gpcolor{color=gp lt color border}
\draw[gp path] (1.012,1.119)--(1.263,1.119);
\draw[gp path] (4.780,1.119)--(4.529,1.119);
\gpcolor{rgb color={0.000,0.000,0.000}}
\node[gp node right,font={\fontsize{10pt}{12pt}\selectfont}] at (0.828,1.119) {20};
\gpcolor{color=gp lt color border}
\draw[gp path] (1.012,1.621)--(1.263,1.621);
\draw[gp path] (4.780,1.621)--(4.529,1.621);
\gpcolor{rgb color={0.000,0.000,0.000}}
\node[gp node right,font={\fontsize{10pt}{12pt}\selectfont}] at (0.828,1.621) {40};
\gpcolor{color=gp lt color border}
\draw[gp path] (1.012,2.124)--(1.263,2.124);
\draw[gp path] (4.780,2.124)--(4.529,2.124);
\gpcolor{rgb color={0.000,0.000,0.000}}
\node[gp node right,font={\fontsize{10pt}{12pt}\selectfont}] at (0.828,2.124) {60};
\gpcolor{color=gp lt color border}
\draw[gp path] (1.012,2.626)--(1.263,2.626);
\draw[gp path] (4.780,2.626)--(4.529,2.626);
\gpcolor{rgb color={0.000,0.000,0.000}}
\node[gp node right,font={\fontsize{10pt}{12pt}\selectfont}] at (0.828,2.626) {80};
\gpcolor{color=gp lt color border}
\draw[gp path] (1.012,3.129)--(1.263,3.129);
\draw[gp path] (4.780,3.129)--(4.529,3.129);
\gpcolor{rgb color={0.000,0.000,0.000}}
\node[gp node right,font={\fontsize{10pt}{12pt}\selectfont}] at (0.828,3.129) {100};
\gpcolor{color=gp lt color border}
\draw[gp path] (1.012,3.631)--(1.263,3.631);
\draw[gp path] (4.780,3.631)--(4.529,3.631);
\gpcolor{rgb color={0.000,0.000,0.000}}
\node[gp node right,font={\fontsize{10pt}{12pt}\selectfont}] at (0.828,3.631) {120};
\gpcolor{color=gp lt color border}
\draw[gp path] (1.012,0.616)--(1.012,0.867);
\draw[gp path] (1.012,3.631)--(1.012,3.380);
\gpcolor{rgb color={0.000,0.000,0.000}}
\node[gp node center,font={\fontsize{10pt}{12pt}\selectfont}] at (1.012,0.308) {6};
\gpcolor{color=gp lt color border}
\draw[gp path] (1.640,0.616)--(1.640,0.867);
\draw[gp path] (1.640,3.631)--(1.640,3.380);
\gpcolor{rgb color={0.000,0.000,0.000}}
\node[gp node center,font={\fontsize{10pt}{12pt}\selectfont}] at (1.640,0.308) {6.5};
\gpcolor{color=gp lt color border}
\draw[gp path] (2.268,0.616)--(2.268,0.867);
\draw[gp path] (2.268,3.631)--(2.268,3.380);
\gpcolor{rgb color={0.000,0.000,0.000}}
\node[gp node center,font={\fontsize{10pt}{12pt}\selectfont}] at (2.268,0.308) {7};
\gpcolor{color=gp lt color border}
\draw[gp path] (2.896,0.616)--(2.896,0.867);
\draw[gp path] (2.896,3.631)--(2.896,3.380);
\gpcolor{rgb color={0.000,0.000,0.000}}
\node[gp node center,font={\fontsize{10pt}{12pt}\selectfont}] at (2.896,0.308) {7.5};
\gpcolor{color=gp lt color border}
\draw[gp path] (3.524,0.616)--(3.524,0.867);
\draw[gp path] (3.524,3.631)--(3.524,3.380);
\gpcolor{rgb color={0.000,0.000,0.000}}
\node[gp node center,font={\fontsize{10pt}{12pt}\selectfont}] at (3.524,0.308) {8};
\gpcolor{color=gp lt color border}
\draw[gp path] (4.152,0.616)--(4.152,0.867);
\draw[gp path] (4.152,3.631)--(4.152,3.380);
\gpcolor{rgb color={0.000,0.000,0.000}}
\node[gp node center,font={\fontsize{10pt}{12pt}\selectfont}] at (4.152,0.308) {8.5};
\gpcolor{color=gp lt color border}
\draw[gp path] (4.780,0.616)--(4.780,0.867);
\draw[gp path] (4.780,3.631)--(4.780,3.380);
\gpcolor{rgb color={0.000,0.000,0.000}}
\node[gp node center,font={\fontsize{10pt}{12pt}\selectfont}] at (4.780,0.308) {9};
\gpcolor{color=gp lt color border}
\draw[gp path] (1.012,3.631)--(1.012,0.616)--(4.780,0.616)--(4.780,3.631)--cycle;
\gpcolor{rgb color={0.000,0.000,1.000}}
\gpsetlinetype{gp lt plot 0}
\gpsetlinewidth{3.00}
\draw[gp path] (1.307,0.616)--(1.355,0.616)--(1.403,0.616)--(1.451,0.616)--(1.499,2.576)%
  --(1.548,2.526)--(1.596,2.978)--(1.644,2.802)--(1.692,3.003)--(1.740,2.827)--(1.788,3.204)%
  --(1.837,2.752)--(1.885,2.701)--(1.933,2.953)--(1.981,3.129)--(2.029,3.053)--(2.077,2.601)%
  --(2.126,2.777)--(2.174,2.701)--(2.222,2.928)--(2.270,2.526)--(2.318,3.103)--(2.366,2.928)%
  --(2.415,3.078)--(2.463,3.330)--(2.511,2.676)--(2.559,2.701)--(2.607,3.129)--(2.655,2.978)%
  --(2.704,2.777)--(2.752,2.953)--(2.800,2.877)--(2.848,2.551)--(2.896,2.601)--(2.944,2.576)%
  --(2.993,2.576)--(3.041,2.727)--(3.089,2.475)--(3.137,2.676)--(3.185,2.576)--(3.233,2.299)%
  --(3.281,2.124)--(3.330,2.651)--(3.378,2.475)--(3.426,2.249)--(3.474,2.048)--(3.522,2.274)%
  --(3.570,2.124)--(3.619,2.098)--(3.667,1.747)--(3.715,2.375)--(3.763,1.772)--(3.811,1.948)%
  --(3.859,2.048)--(3.908,0.892)--(3.956,0.616)--(4.004,0.616)--(4.052,0.616)--(4.100,0.616)%
  --(4.148,0.616)--(4.197,0.616)--(4.245,0.616)--(4.293,0.616);
\gpcolor{rgb color={0.000,0.502,0.000}}
\gpsetlinetype{gp lt plot 5}
\draw[gp path] (1.307,0.641)--(1.355,0.717)--(1.403,1.144)--(1.451,1.621)--(1.499,1.722)%
  --(1.548,2.576)--(1.596,2.928)--(1.644,2.576)--(1.692,2.727)--(1.740,2.701)--(1.788,2.752)%
  --(1.837,2.651)--(1.885,3.129)--(1.933,2.852)--(1.981,2.978)--(2.029,2.802)--(2.077,3.078)%
  --(2.126,2.626)--(2.174,2.500)--(2.222,3.129)--(2.270,2.701)--(2.318,3.154)--(2.366,2.928)%
  --(2.415,2.877)--(2.463,3.455)--(2.511,3.279)--(2.559,2.902)--(2.607,2.601)--(2.655,2.727)%
  --(2.704,2.701)--(2.752,2.727)--(2.800,2.978)--(2.848,2.576)--(2.896,2.626)--(2.944,2.676)%
  --(2.993,2.701)--(3.041,2.626)--(3.089,2.551)--(3.137,2.526)--(3.185,2.475)--(3.233,2.249)%
  --(3.281,2.676)--(3.330,2.500)--(3.378,2.475)--(3.426,2.048)--(3.474,2.500)--(3.522,1.998)%
  --(3.570,2.199)--(3.619,2.073)--(3.667,1.923)--(3.715,1.546)--(3.763,1.671)--(3.811,1.722)%
  --(3.859,1.420)--(3.908,1.043)--(3.956,1.144)--(4.004,0.842)--(4.052,0.817)--(4.100,0.742)%
  --(4.148,0.691)--(4.197,0.666)--(4.245,0.616)--(4.293,0.641);
\gpcolor{rgb color={1.000,0.000,0.000}}
\gpsetlinetype{gp lt plot 2}
\draw[gp path] (2.797,0.616)--(2.797,3.455);
\gpsetpointsize{8.00}
\gppoint{gp mark 8}{(2.797,0.616)}
\gppoint{gp mark 8}{(2.797,3.455)}
\gpcolor{rgb color={0.000,0.749,0.749}}
\gppoint{gp mark 2}{(2.589,2.036)}
\gpcolor{rgb color={0.749,0.000,0.749}}
\gppoint{gp mark 2}{(2.589,2.036)}
%% coordinates of the plot area
\gpdefrectangularnode{gp plot 1}{\pgfpoint{1.012cm}{0.616cm}}{\pgfpoint{4.780cm}{3.631cm}}
\end{tikzpicture}
%% gnuplot variables
  }
  \subfigure[$\threshold = 0.1, \ntrain=10^3$]{
    \begin{tikzpicture}[gnuplot]
%% generated with GNUPLOT 4.6p1 (Lua 5.1; terminal rev. 99, script rev. 100)
%% Wed 09 Jan 2013 01:57:46 PM CET
\path (0.000,0.000) rectangle (5.333,4.000);
\gpcolor{color=gp lt color border}
\gpsetlinetype{gp lt border}
\gpsetlinewidth{0.50}
\draw[gp path] (1.012,0.616)--(1.263,0.616);
\draw[gp path] (4.780,0.616)--(4.529,0.616);
\gpcolor{rgb color={0.000,0.000,0.000}}
\node[gp node right,font={\fontsize{10pt}{12pt}\selectfont}] at (0.828,0.616) {0};
\gpcolor{color=gp lt color border}
\draw[gp path] (1.012,1.219)--(1.263,1.219);
\draw[gp path] (4.780,1.219)--(4.529,1.219);
\gpcolor{rgb color={0.000,0.000,0.000}}
\node[gp node right,font={\fontsize{10pt}{12pt}\selectfont}] at (0.828,1.219) {20};
\gpcolor{color=gp lt color border}
\draw[gp path] (1.012,1.822)--(1.263,1.822);
\draw[gp path] (4.780,1.822)--(4.529,1.822);
\gpcolor{rgb color={0.000,0.000,0.000}}
\node[gp node right,font={\fontsize{10pt}{12pt}\selectfont}] at (0.828,1.822) {40};
\gpcolor{color=gp lt color border}
\draw[gp path] (1.012,2.425)--(1.263,2.425);
\draw[gp path] (4.780,2.425)--(4.529,2.425);
\gpcolor{rgb color={0.000,0.000,0.000}}
\node[gp node right,font={\fontsize{10pt}{12pt}\selectfont}] at (0.828,2.425) {60};
\gpcolor{color=gp lt color border}
\draw[gp path] (1.012,3.028)--(1.263,3.028);
\draw[gp path] (4.780,3.028)--(4.529,3.028);
\gpcolor{rgb color={0.000,0.000,0.000}}
\node[gp node right,font={\fontsize{10pt}{12pt}\selectfont}] at (0.828,3.028) {80};
\gpcolor{color=gp lt color border}
\draw[gp path] (1.012,3.631)--(1.263,3.631);
\draw[gp path] (4.780,3.631)--(4.529,3.631);
\gpcolor{rgb color={0.000,0.000,0.000}}
\node[gp node right,font={\fontsize{10pt}{12pt}\selectfont}] at (0.828,3.631) {100};
\gpcolor{color=gp lt color border}
\draw[gp path] (1.012,0.616)--(1.012,0.867);
\draw[gp path] (1.012,3.631)--(1.012,3.380);
\gpcolor{rgb color={0.000,0.000,0.000}}
\node[gp node center,font={\fontsize{10pt}{12pt}\selectfont}] at (1.012,0.308) {6};
\gpcolor{color=gp lt color border}
\draw[gp path] (1.640,0.616)--(1.640,0.867);
\draw[gp path] (1.640,3.631)--(1.640,3.380);
\gpcolor{rgb color={0.000,0.000,0.000}}
\node[gp node center,font={\fontsize{10pt}{12pt}\selectfont}] at (1.640,0.308) {6.5};
\gpcolor{color=gp lt color border}
\draw[gp path] (2.268,0.616)--(2.268,0.867);
\draw[gp path] (2.268,3.631)--(2.268,3.380);
\gpcolor{rgb color={0.000,0.000,0.000}}
\node[gp node center,font={\fontsize{10pt}{12pt}\selectfont}] at (2.268,0.308) {7};
\gpcolor{color=gp lt color border}
\draw[gp path] (2.896,0.616)--(2.896,0.867);
\draw[gp path] (2.896,3.631)--(2.896,3.380);
\gpcolor{rgb color={0.000,0.000,0.000}}
\node[gp node center,font={\fontsize{10pt}{12pt}\selectfont}] at (2.896,0.308) {7.5};
\gpcolor{color=gp lt color border}
\draw[gp path] (3.524,0.616)--(3.524,0.867);
\draw[gp path] (3.524,3.631)--(3.524,3.380);
\gpcolor{rgb color={0.000,0.000,0.000}}
\node[gp node center,font={\fontsize{10pt}{12pt}\selectfont}] at (3.524,0.308) {8};
\gpcolor{color=gp lt color border}
\draw[gp path] (4.152,0.616)--(4.152,0.867);
\draw[gp path] (4.152,3.631)--(4.152,3.380);
\gpcolor{rgb color={0.000,0.000,0.000}}
\node[gp node center,font={\fontsize{10pt}{12pt}\selectfont}] at (4.152,0.308) {8.5};
\gpcolor{color=gp lt color border}
\draw[gp path] (4.780,0.616)--(4.780,0.867);
\draw[gp path] (4.780,3.631)--(4.780,3.380);
\gpcolor{rgb color={0.000,0.000,0.000}}
\node[gp node center,font={\fontsize{10pt}{12pt}\selectfont}] at (4.780,0.308) {9};
\gpcolor{color=gp lt color border}
\draw[gp path] (1.012,3.631)--(1.012,0.616)--(4.780,0.616)--(4.780,3.631)--cycle;
\gpcolor{rgb color={0.000,0.000,1.000}}
\gpsetlinetype{gp lt plot 0}
\gpsetlinewidth{3.00}
\draw[gp path] (2.358,0.616)--(2.398,0.616)--(2.439,0.616)--(2.480,0.616)--(2.521,0.616)%
  --(2.562,0.616)--(2.603,2.033)--(2.644,3.088)--(2.685,2.938)--(2.726,2.696)--(2.766,1.792)%
  --(2.807,0.616)--(2.848,0.616)--(2.889,0.616)--(2.930,0.616)--(2.971,0.616)--(3.012,0.616);
\gpcolor{rgb color={0.000,0.502,0.000}}
\gpsetlinetype{gp lt plot 5}
\draw[gp path] (2.358,0.767)--(2.398,0.767)--(2.439,0.827)--(2.480,0.948)--(2.521,0.918)%
  --(2.562,1.490)--(2.603,1.581)--(2.644,1.611)--(2.685,1.792)--(2.726,1.551)--(2.766,1.882)%
  --(2.807,1.460)--(2.848,1.279)--(2.889,0.887)--(2.930,0.797)--(2.971,0.737)--(3.012,0.646);
\gpcolor{rgb color={1.000,0.000,0.000}}
\gpsetlinetype{gp lt plot 2}
\draw[gp path] (2.797,0.616)--(2.797,3.088);
\gpsetpointsize{8.00}
\gppoint{gp mark 8}{(2.797,0.616)}
\gppoint{gp mark 8}{(2.797,3.088)}
\gpcolor{rgb color={0.000,0.749,0.749}}
\gppoint{gp mark 2}{(2.687,1.852)}
\gpcolor{rgb color={0.749,0.000,0.749}}
\gppoint{gp mark 2}{(2.687,1.852)}
%% coordinates of the plot area
\gpdefrectangularnode{gp plot 1}{\pgfpoint{1.012cm}{0.616cm}}{\pgfpoint{4.780cm}{3.631cm}}
\end{tikzpicture}
%% gnuplot variables
  }
  % \subfigure[fewer rollouts]{
  % \input{results/Pendulum/1.0thr_100trj_10000sam_1000train.tex}
  % }
  %\subfigure[$\threshold = 1, \ntrain=10^2$]{
  %  \input{results/Pendulum/1.0thr_1000trj_10000sam_100train.tex}
  %}
  %\subfigure[$\threshold = 0.1, \ntrain=10^2$]{
  %  \input{results/Pendulum/0.1thr_1000trj_10000sam_100train.tex}
  %}
  %\subfigure[$\threshold = 1, \ntrain=10^2$]{
  % \input{results/Pendulum/1.0thr_100trj_10000sam_100train.tex}
  % }
  %   \subfigure[fewer rollouts, smaller threshold]{
  %   \input{results/Pendulum/0.1thr_100trj_10000sam_1000train.tex}
  % }
  \caption{\textbf{Pendulum value distribution}. In both cases, $\nsamples = 10^4$ model samples are drawn from the prior and $\nrollouts = 10^3$ rollouts are performed for each model sample.  The vertical dashed line shows the actual value of the policy. The solid and dot-dashed lines show the histograms of real and estimated values of the original policy in the sampled environment.  The solid line shows the value estimated using $10^4$ rollouts. The dot-dashed line shows the value estimated in the run itself, with $\ntrajectories$ rollouts per sample. The $\times$ shows the expected value, averaged over the accepted samples. It can be seen that, while a smaller threshold can result in better accuracy, many fewer samples are accepted. }
  \label{fig:pendulum-error}
\end{figure*}

\section{ABC reinforcement learning}
\label{sec:abc-rl}
We now present a simple algorithm for ABC reinforcement learning, based on the ideas explained in the previous section. For any given set of observations and policies, we draw a number of sample environments from the prior distribution. For each environment, we execute the relevant policy and calculate the appropriate statistics. If these are close enough to the observed statistic, the sample is accepted. The next step is to find a good policy for the sampled simulator. As we can draw an arbitrary number of rollouts in the simulator, any type of approximate dynamic programming algorithm can be used. In our experiments, we used LSPI~\citep{lagoudakis2003least}, which is simple to program and effective. The hope is that if the approximate posterior sampling is reasonable, then we can take advantage of our prior knowledge of the environment class, to learn a good policy with less data, at the expense of additional computation.

\begin{algorithm}[ht]
  \begin{algorithmic}
    \STATE \textbf{parameters}  $\MDPs$, $\bel$, $h$, $\pol$, $f$
    \STATE $\tau = \{|\widehat{M}|=1\}$
    \STATE $\hat{\mdp} = \texttt{ABC-RL-Sample}(\MDPs, \bel, h, \pol, f, \tau)$
    \STATE \textbf{return} $\hat{\pol} \approx \argmax_\pol \E_{\hat{\mdp}}^{\pol} U$
  \end{algorithmic}
  \caption{ABC-RL}
  \label{alg:abc-rl-policy}
\end{algorithm}
A sketch of the algorithm is shown in Alg.\ref{alg:abc-rl-policy}. 
This has a number of additional parameters that need to be discussed. The most important is the stopping condition $\tau$. The simplest idea, which we use in this paper, is to stop when a single model $\hat{\mdp}$ has been generated by $\texttt{ABC-RL-Sample}$.

Then an (approximate) optimal policy for the sampled model $\hat{\mdp}$ can be found via an exact (or approximate)  dynamic programming algorithm.  This simplifies the optimisation step significantly, as otherwise it would be necessary to optimise over multiple models. This particular version of the algorithm can be seen as an ABC variant of Thompson sampling~\citep{thompson1933lou,strens2000bayesian}.

The exact algorithm to use for the policy optimisation depends largely upon the class of simulators we have. In principle any type of environment can be handled, as long as a simulation-based approximation method can be used to discover a good policy. In \emph{extremis}, direct policy search may be used.
However, in the work presented in this paper, we limit ourselves to continuous-state Markov decision processes, for which numerous efficient ADP algorithms exist. 

\subsection{ABC-LSPI}
\label{sec:abc-lspi}
Let us consider the class of continuous-state, discrete-action
Markov decision processes (MDPs). Then, a number of sample-based ADP
algorithms can be used to find good policies, such as 
fitted Q-iteration (FQI)~\citep{Ernst:TreeRL} and least-square
policy iteration (LSPI)~\citep{lagoudakis2003least}, which we use herein.

Since we take an arbitrary number of trajectories from the sampled
MDP, an important algorithmic parameter is the number of rollouts
$\nrollouts$ to draw. Higher values lead to better approximations,
at the expense of additional computation. Finally, since LSPI uses a linear value function\footnote{The value function $V(s)$ is simply the expected utility conditioned on the system state $s$. We omit details as this is not  necessary to understand the framework proposed.} approximation, it is necessary to select an appropriate basis  for the fit to be good.
% A higher-dimensional basis can lead to better fits, but generally requires more rollouts.

The computational complexity of ABC-LSPI depends on the quality of approximation we wish to achieve and on the number of samples required to sample a model with statistics $\threshold$-close to those of the data. To reduce computation, if $\nsamples$ models have been generated without one being accepted, we double $\epsilon$ and call \texttt{ABC-RL-Sample} again.
\section{Experiments}
\label{sec:experiments}
We performed some experiments to investigate the viability of ABC-RL, with all algorithms implemented using~\cite{beliefbox}. In these, we compared  ABC-LSPI to LSPI. The intuition is that, if ABC can find a good simulator, then we can perform a much better estimation of the value function by drawing a large number of samples from the simulator, rather than estimating the value function directly from the observations.

\subsection{Domains}
\label{sec:domains}
We consider two domains to illustrate ABC-RL. In both of these domains, we have access to a set of parametrised simulators $\MDPs = \cset{\mdp_\theta}{\theta \in \Theta}$ for the domains. However, we do not know the true parameters $\theta^* \in \Theta$ of the domains. 
%In our experiments, we use \emph{generalised} domains, meaning that the $\theta^*$ is not fixed but drawn from some distribution.
For ABC, sampled parameters $\theta^{(k)}$ are drawn from a uniform distribution $\Uniform(\Theta)$, with $\Theta = \cset{\theta \in \Reals^n}{\theta_i \in [\frac{1}{2}\theta_i^*, \frac{3}{2}\theta_i^*]}$.

\paragraph{Mountain car} This is a generalised version of the mountain car domain described in~\citet{Sutton+Barto:1998}. The goal is to bring a car to the top of a hill. The problem has 7 parameters: upper and lower bounds on the horizontal position of the car, upper and lower bounds on the car's velocity, maximum acceleration, gravity, and finally the amount of uniform noise present.
The real environment parameters are $\theta^* = (0.5, -1.2, 0.07, -0.07, 0.001, 0.0025, 0.2)$.
 In this problem, the goal is to reach the right-most horizontal position.
The observation consists of the horizontal position and velocity and the reward is $-1$ at every step until the goal is reached.

\paragraph{Pendulum} This is a generalised version of the pendulum domain~\citep{Sutton+Barto:1998}, but without boundaries. The goal of the agent in this environment is to maintain a pendulum upright, using a controller that can switch actions every $0.1s$. The problem has 6 parameters: the pendulum mass, the cart mass, the pendulum length, the gravity, the amount of uniform noise, and the simulation time interval. In this environment, the reward is $+1$ for every step where the pendulum is balanced.
The actual environment parameters are $\theta^* = (2.0, 8.0, 0.5, 9.8, 0.01, 0.01)$.

\subsection{Results}  
\label{sec:results}
\begin{figure}[htb]
  \centering
  \subfigure[Mountain Car]{
    \begin{tikzpicture}[gnuplot]
%% generated with GNUPLOT 4.6p0 (Lua 5.1; terminal rev. 99, script rev. 100)
%% Tue 12 Feb 2013 21:36:55 CET
\path (0.000,0.000) rectangle (8.000,6.000);
\gpfill{rgb color={1.000,1.000,1.000}} (1.320,0.985)--(7.447,0.985)--(7.447,5.631)--(1.320,5.631)--cycle;
\gpcolor{color=gp lt color border}
\gpsetlinetype{gp lt border}
\gpsetlinewidth{1.00}
\draw[gp path] (1.320,0.985)--(1.320,5.631)--(7.447,5.631)--(7.447,0.985)--cycle;
\gpsetlinewidth{0.50}
\draw[gp path] (1.320,0.985)--(1.571,0.985);
\draw[gp path] (7.447,0.985)--(7.196,0.985);
\gpcolor{rgb color={0.000,0.000,0.000}}
\node[gp node right,font={\fontsize{10pt}{12pt}\selectfont}] at (1.136,0.985) {-80};
\gpcolor{color=gp lt color border}
\draw[gp path] (1.320,1.914)--(1.571,1.914);
\draw[gp path] (7.447,1.914)--(7.196,1.914);
\gpcolor{rgb color={0.000,0.000,0.000}}
\node[gp node right,font={\fontsize{10pt}{12pt}\selectfont}] at (1.136,1.914) {-70};
\gpcolor{color=gp lt color border}
\draw[gp path] (1.320,2.843)--(1.571,2.843);
\draw[gp path] (7.447,2.843)--(7.196,2.843);
\gpcolor{rgb color={0.000,0.000,0.000}}
\node[gp node right,font={\fontsize{10pt}{12pt}\selectfont}] at (1.136,2.843) {-60};
\gpcolor{color=gp lt color border}
\draw[gp path] (1.320,3.773)--(1.571,3.773);
\draw[gp path] (7.447,3.773)--(7.196,3.773);
\gpcolor{rgb color={0.000,0.000,0.000}}
\node[gp node right,font={\fontsize{10pt}{12pt}\selectfont}] at (1.136,3.773) {-50};
\gpcolor{color=gp lt color border}
\draw[gp path] (1.320,4.702)--(1.571,4.702);
\draw[gp path] (7.447,4.702)--(7.196,4.702);
\gpcolor{rgb color={0.000,0.000,0.000}}
\node[gp node right,font={\fontsize{10pt}{12pt}\selectfont}] at (1.136,4.702) {-40};
\gpcolor{color=gp lt color border}
\draw[gp path] (1.320,5.631)--(1.571,5.631);
\draw[gp path] (7.447,5.631)--(7.196,5.631);
\gpcolor{rgb color={0.000,0.000,0.000}}
\node[gp node right,font={\fontsize{10pt}{12pt}\selectfont}] at (1.136,5.631) {-30};
\gpcolor{color=gp lt color border}
\draw[gp path] (1.320,0.985)--(1.320,1.236);
\draw[gp path] (1.320,5.631)--(1.320,5.380);
\gpcolor{rgb color={0.000,0.000,0.000}}
\node[gp node center,font={\fontsize{10pt}{12pt}\selectfont}] at (1.320,0.677) {$10^{0}$};
\gpcolor{color=gp lt color border}
\draw[gp path] (3.362,0.985)--(3.362,1.236);
\draw[gp path] (3.362,5.631)--(3.362,5.380);
\gpcolor{rgb color={0.000,0.000,0.000}}
\node[gp node center,font={\fontsize{10pt}{12pt}\selectfont}] at (3.362,0.677) {$10^{1}$};
\gpcolor{color=gp lt color border}
\draw[gp path] (5.405,0.985)--(5.405,1.236);
\draw[gp path] (5.405,5.631)--(5.405,5.380);
\gpcolor{rgb color={0.000,0.000,0.000}}
\node[gp node center,font={\fontsize{10pt}{12pt}\selectfont}] at (5.405,0.677) {$10^{2}$};
\gpcolor{color=gp lt color border}
\draw[gp path] (7.447,0.985)--(7.447,1.236);
\draw[gp path] (7.447,5.631)--(7.447,5.380);
\gpcolor{rgb color={0.000,0.000,0.000}}
\node[gp node center,font={\fontsize{10pt}{12pt}\selectfont}] at (7.447,0.677) {$10^{3}$};
\gpcolor{color=gp lt color border}
\draw[gp path] (1.320,5.631)--(1.320,0.985)--(7.447,0.985)--(7.447,5.631)--cycle;
\gpcolor{rgb color={0.000,0.000,0.000}}
\node[gp node center,rotate=90,font={\fontsize{10pt}{12pt}\selectfont}] at (0.246,3.308) {value};
\node[gp node center,font={\fontsize{10pt}{12pt}\selectfont}] at (4.383,0.215) {trajectories};
\gpfill{rgb color={1.000,0.000,0.000},opacity=0.10} (1.320,2.194)--(1.320,2.194)--(1.935,2.777)--(2.748,3.471)%
    --(3.362,3.391)--(3.977,4.031)--(4.790,4.030)--(5.405,4.023)--(6.019,4.406)%
    --(6.832,4.395)--(7.447,4.331)--(7.447,3.838)--(6.832,3.953)--(6.019,3.973)%
    --(5.405,3.461)--(4.790,3.488)--(3.977,3.494)--(3.362,2.776)--(2.748,2.875)%
    --(1.935,2.104)--(1.320,1.610)--cycle;
\gpcolor{rgb color={1.000,0.000,0.000}}
\gpsetlinetype{gp lt plot 0}
\gpsetlinewidth{1.00}
\draw[gp path] (1.320,2.194)--(1.935,2.777)--(2.748,3.471)--(3.362,3.391)--(3.977,4.031)%
  --(4.790,4.030)--(5.405,4.023)--(6.019,4.406)--(6.832,4.395)--(7.447,4.331)--(7.447,3.838)%
  --(6.832,3.953)--(6.019,3.973)--(5.405,3.461)--(4.790,3.488)--(3.977,3.494)--(3.362,2.776)%
  --(2.748,2.875)--(1.935,2.104)--(1.320,1.610);
\gpsetlinewidth{0.50}
\draw[gp path] (1.320,2.194)--(1.935,2.777)--(2.748,3.471)--(3.362,3.391)--(3.977,4.031)%
  --(4.790,4.030)--(5.405,4.023)--(6.019,4.406)--(6.832,4.395)--(7.447,4.331)--(7.447,3.838)%
  --(6.832,3.953)--(6.019,3.973)--(5.405,3.461)--(4.790,3.488)--(3.977,3.494)--(3.362,2.776)%
  --(2.748,2.875)--(1.935,2.104)--(1.320,1.610)--cycle;
\gpfill{rgb color={0.000,0.000,1.000},opacity=0.10} (1.320,1.690)--(1.320,1.690)--(1.935,1.468)--(2.748,2.177)%
    --(3.362,2.606)--(3.977,3.207)--(4.790,4.371)--(5.405,4.670)--(6.019,4.795)%
    --(6.832,4.876)--(7.447,4.892)--(7.447,4.860)--(6.832,4.843)--(6.019,4.728)%
    --(5.405,4.532)--(4.790,4.065)--(3.977,2.580)--(3.362,1.956)--(2.748,1.590)%
    --(1.935,1.007)--(1.320,1.141)--cycle;
\gpcolor{rgb color={0.000,0.000,1.000}}
\gpsetlinetype{gp lt plot 2}
\gpsetlinewidth{1.00}
\draw[gp path] (1.320,1.690)--(1.935,1.468)--(2.748,2.177)--(3.362,2.606)--(3.977,3.207)%
  --(4.790,4.371)--(5.405,4.670)--(6.019,4.795)--(6.832,4.876)--(7.447,4.892)--(7.447,4.860)%
  --(6.832,4.843)--(6.019,4.728)--(5.405,4.532)--(4.790,4.065)--(3.977,2.580)--(3.362,1.956)%
  --(2.748,1.590)--(1.935,1.007)--(1.320,1.141);
\gpsetlinetype{gp lt plot 0}
\gpsetlinewidth{0.50}
\draw[gp path] (1.320,1.690)--(1.935,1.468)--(2.748,2.177)--(3.362,2.606)--(3.977,3.207)%
  --(4.790,4.371)--(5.405,4.670)--(6.019,4.795)--(6.832,4.876)--(7.447,4.892)--(7.447,4.860)%
  --(6.832,4.843)--(6.019,4.728)--(5.405,4.532)--(4.790,4.065)--(3.977,2.580)--(3.362,1.956)%
  --(2.748,1.590)--(1.935,1.007)--(1.320,1.141)--cycle;
\gpcolor{color=gp lt color border}
\node[gp node right,font={\fontsize{10pt}{12pt}\selectfont}] at (2.240,5.297) {ABC};
\gpcolor{rgb color={1.000,0.000,0.000}}
\gpsetlinewidth{3.00}
\draw[gp path] (2.424,5.297)--(3.340,5.297);
\draw[gp path] (1.320,1.884)--(1.935,2.441)--(2.748,3.185)--(3.362,3.100)--(3.977,3.767)%
  --(4.790,3.785)--(5.405,3.733)--(6.019,4.205)--(6.832,4.191)--(7.447,4.088);
\gpsetpointsize{8.00}
\gppoint{gp mark 6}{(1.320,1.884)}
\gppoint{gp mark 6}{(1.935,2.441)}
\gppoint{gp mark 6}{(2.748,3.185)}
\gppoint{gp mark 6}{(3.362,3.100)}
\gppoint{gp mark 6}{(3.977,3.767)}
\gppoint{gp mark 6}{(4.790,3.785)}
\gppoint{gp mark 6}{(5.405,3.733)}
\gppoint{gp mark 6}{(6.019,4.205)}
\gppoint{gp mark 6}{(6.832,4.191)}
\gppoint{gp mark 6}{(7.447,4.088)}
\gppoint{gp mark 6}{(2.882,5.297)}
\gpcolor{color=gp lt color border}
\node[gp node right,font={\fontsize{10pt}{12pt}\selectfont}] at (2.240,4.989) {LSPI};
\gpcolor{rgb color={0.000,0.000,1.000}}
\gpsetlinetype{gp lt plot 5}
\draw[gp path] (2.424,4.989)--(3.340,4.989);
\draw[gp path] (1.320,1.406)--(1.935,1.239)--(2.748,1.880)--(3.362,2.279)--(3.977,2.900)%
  --(4.790,4.235)--(5.405,4.606)--(6.019,4.762)--(6.832,4.860)--(7.447,4.875);
\gppoint{gp mark 2}{(1.320,1.406)}
\gppoint{gp mark 2}{(1.935,1.239)}
\gppoint{gp mark 2}{(2.748,1.880)}
\gppoint{gp mark 2}{(3.362,2.279)}
\gppoint{gp mark 2}{(3.977,2.900)}
\gppoint{gp mark 2}{(4.790,4.235)}
\gppoint{gp mark 2}{(5.405,4.606)}
\gppoint{gp mark 2}{(6.019,4.762)}
\gppoint{gp mark 2}{(6.832,4.860)}
\gppoint{gp mark 2}{(7.447,4.875)}
\gppoint{gp mark 2}{(2.882,4.989)}
%% coordinates of the plot area
\gpdefrectangularnode{gp plot 1}{\pgfpoint{1.320cm}{0.985cm}}{\pgfpoint{7.447cm}{5.631cm}}
\end{tikzpicture}
%% gnuplot variables
    \label{fig:offline-mountain-car}
  }
  \subfigure[Pendulum]{
    \begin{tikzpicture}[gnuplot]
%% generated with GNUPLOT 4.6p0 (Lua 5.1; terminal rev. 99, script rev. 100)
%% Tue 12 Feb 2013 21:37:27 CET
\path (0.000,0.000) rectangle (8.000,6.000);
\gpfill{rgb color={1.000,1.000,1.000}} (1.320,0.985)--(7.447,0.985)--(7.447,5.631)--(1.320,5.631)--cycle;
\gpcolor{color=gp lt color border}
\gpsetlinetype{gp lt border}
\gpsetlinewidth{1.00}
\draw[gp path] (1.320,0.985)--(1.320,5.631)--(7.447,5.631)--(7.447,0.985)--cycle;
\gpsetlinewidth{0.50}
\draw[gp path] (1.320,0.985)--(1.571,0.985);
\draw[gp path] (7.447,0.985)--(7.196,0.985);
\gpcolor{rgb color={0.000,0.000,0.000}}
\node[gp node right,font={\fontsize{10pt}{12pt}\selectfont}] at (1.136,0.985) {0};
\gpcolor{color=gp lt color border}
\draw[gp path] (1.320,1.914)--(1.571,1.914);
\draw[gp path] (7.447,1.914)--(7.196,1.914);
\gpcolor{rgb color={0.000,0.000,0.000}}
\node[gp node right,font={\fontsize{10pt}{12pt}\selectfont}] at (1.136,1.914) {20};
\gpcolor{color=gp lt color border}
\draw[gp path] (1.320,2.843)--(1.571,2.843);
\draw[gp path] (7.447,2.843)--(7.196,2.843);
\gpcolor{rgb color={0.000,0.000,0.000}}
\node[gp node right,font={\fontsize{10pt}{12pt}\selectfont}] at (1.136,2.843) {40};
\gpcolor{color=gp lt color border}
\draw[gp path] (1.320,3.773)--(1.571,3.773);
\draw[gp path] (7.447,3.773)--(7.196,3.773);
\gpcolor{rgb color={0.000,0.000,0.000}}
\node[gp node right,font={\fontsize{10pt}{12pt}\selectfont}] at (1.136,3.773) {60};
\gpcolor{color=gp lt color border}
\draw[gp path] (1.320,4.702)--(1.571,4.702);
\draw[gp path] (7.447,4.702)--(7.196,4.702);
\gpcolor{rgb color={0.000,0.000,0.000}}
\node[gp node right,font={\fontsize{10pt}{12pt}\selectfont}] at (1.136,4.702) {80};
\gpcolor{color=gp lt color border}
\draw[gp path] (1.320,5.631)--(1.571,5.631);
\draw[gp path] (7.447,5.631)--(7.196,5.631);
\gpcolor{rgb color={0.000,0.000,0.000}}
\node[gp node right,font={\fontsize{10pt}{12pt}\selectfont}] at (1.136,5.631) {100};
\gpcolor{color=gp lt color border}
\draw[gp path] (1.320,0.985)--(1.320,1.236);
\draw[gp path] (1.320,5.631)--(1.320,5.380);
\gpcolor{rgb color={0.000,0.000,0.000}}
\node[gp node center,font={\fontsize{10pt}{12pt}\selectfont}] at (1.320,0.677) {$10^{0}$};
\gpcolor{color=gp lt color border}
\draw[gp path] (3.362,0.985)--(3.362,1.236);
\draw[gp path] (3.362,5.631)--(3.362,5.380);
\gpcolor{rgb color={0.000,0.000,0.000}}
\node[gp node center,font={\fontsize{10pt}{12pt}\selectfont}] at (3.362,0.677) {$10^{1}$};
\gpcolor{color=gp lt color border}
\draw[gp path] (5.405,0.985)--(5.405,1.236);
\draw[gp path] (5.405,5.631)--(5.405,5.380);
\gpcolor{rgb color={0.000,0.000,0.000}}
\node[gp node center,font={\fontsize{10pt}{12pt}\selectfont}] at (5.405,0.677) {$10^{2}$};
\gpcolor{color=gp lt color border}
\draw[gp path] (7.447,0.985)--(7.447,1.236);
\draw[gp path] (7.447,5.631)--(7.447,5.380);
\gpcolor{rgb color={0.000,0.000,0.000}}
\node[gp node center,font={\fontsize{10pt}{12pt}\selectfont}] at (7.447,0.677) {$10^{3}$};
\gpcolor{color=gp lt color border}
\draw[gp path] (1.320,5.631)--(1.320,0.985)--(7.447,0.985)--(7.447,5.631)--cycle;
\gpcolor{rgb color={0.000,0.000,0.000}}
\node[gp node center,rotate=90,font={\fontsize{10pt}{12pt}\selectfont}] at (0.246,3.308) {value};
\node[gp node center,font={\fontsize{10pt}{12pt}\selectfont}] at (4.383,0.215) {trajectories};
\gpfill{rgb color={1.000,0.000,0.000},opacity=0.10} (1.320,1.398)--(1.320,1.398)--(1.935,1.420)--(2.748,1.572)%
    --(3.362,2.493)--(3.977,3.845)--(4.790,5.242)--(5.405,5.600)--(6.019,5.526)%
    --(6.832,5.591)--(7.447,5.569)--(7.447,5.328)--(6.832,5.399)--(6.019,5.192)%
    --(5.405,5.383)--(4.790,4.712)--(3.977,3.135)--(3.362,1.971)--(2.748,1.410)%
    --(1.935,1.342)--(1.320,1.312)--cycle;
\gpcolor{rgb color={1.000,0.000,0.000}}
\gpsetlinetype{gp lt plot 0}
\gpsetlinewidth{1.00}
\draw[gp path] (1.320,1.398)--(1.935,1.420)--(2.748,1.572)--(3.362,2.493)--(3.977,3.845)%
  --(4.790,5.242)--(5.405,5.600)--(6.019,5.526)--(6.832,5.591)--(7.447,5.569)--(7.447,5.328)%
  --(6.832,5.399)--(6.019,5.192)--(5.405,5.383)--(4.790,4.712)--(3.977,3.135)--(3.362,1.971)%
  --(2.748,1.410)--(1.935,1.342)--(1.320,1.312);
\gpsetlinewidth{0.50}
\draw[gp path] (1.320,1.398)--(1.935,1.420)--(2.748,1.572)--(3.362,2.493)--(3.977,3.845)%
  --(4.790,5.242)--(5.405,5.600)--(6.019,5.526)--(6.832,5.591)--(7.447,5.569)--(7.447,5.328)%
  --(6.832,5.399)--(6.019,5.192)--(5.405,5.383)--(4.790,4.712)--(3.977,3.135)--(3.362,1.971)%
  --(2.748,1.410)--(1.935,1.342)--(1.320,1.312)--cycle;
\gpfill{rgb color={0.000,0.000,1.000},opacity=0.10} (1.320,1.371)--(1.320,1.371)--(1.935,1.374)--(2.748,1.461)%
    --(3.362,1.635)--(3.977,1.885)--(4.790,2.460)--(5.405,2.891)--(6.019,3.129)%
    --(6.832,3.407)--(7.447,3.000)--(7.447,2.308)--(6.832,2.625)--(6.019,2.374)%
    --(5.405,2.223)--(4.790,1.828)--(3.977,1.542)--(3.362,1.435)--(2.748,1.381)%
    --(1.935,1.297)--(1.320,1.289)--cycle;
\gpcolor{rgb color={0.000,0.000,1.000}}
\gpsetlinetype{gp lt plot 2}
\gpsetlinewidth{1.00}
\draw[gp path] (1.320,1.371)--(1.935,1.374)--(2.748,1.461)--(3.362,1.635)--(3.977,1.885)%
  --(4.790,2.460)--(5.405,2.891)--(6.019,3.129)--(6.832,3.407)--(7.447,3.000)--(7.447,2.308)%
  --(6.832,2.625)--(6.019,2.374)--(5.405,2.223)--(4.790,1.828)--(3.977,1.542)--(3.362,1.435)%
  --(2.748,1.381)--(1.935,1.297)--(1.320,1.289);
\gpsetlinetype{gp lt plot 0}
\gpsetlinewidth{0.50}
\draw[gp path] (1.320,1.371)--(1.935,1.374)--(2.748,1.461)--(3.362,1.635)--(3.977,1.885)%
  --(4.790,2.460)--(5.405,2.891)--(6.019,3.129)--(6.832,3.407)--(7.447,3.000)--(7.447,2.308)%
  --(6.832,2.625)--(6.019,2.374)--(5.405,2.223)--(4.790,1.828)--(3.977,1.542)--(3.362,1.435)%
  --(2.748,1.381)--(1.935,1.297)--(1.320,1.289)--cycle;
\gpcolor{color=gp lt color border}
\node[gp node right,font={\fontsize{10pt}{12pt}\selectfont}] at (2.240,5.297) {ABC};
\gpcolor{rgb color={1.000,0.000,0.000}}
\gpsetlinewidth{3.00}
\draw[gp path] (2.424,5.297)--(3.340,5.297);
\draw[gp path] (1.320,1.355)--(1.935,1.380)--(2.748,1.493)--(3.362,2.211)--(3.977,3.484)%
  --(4.790,5.001)--(5.405,5.504)--(6.019,5.370)--(6.832,5.516)--(7.447,5.463);
\gpsetpointsize{8.00}
\gppoint{gp mark 6}{(1.320,1.355)}
\gppoint{gp mark 6}{(1.935,1.380)}
\gppoint{gp mark 6}{(2.748,1.493)}
\gppoint{gp mark 6}{(3.362,2.211)}
\gppoint{gp mark 6}{(3.977,3.484)}
\gppoint{gp mark 6}{(4.790,5.001)}
\gppoint{gp mark 6}{(5.405,5.504)}
\gppoint{gp mark 6}{(6.019,5.370)}
\gppoint{gp mark 6}{(6.832,5.516)}
\gppoint{gp mark 6}{(7.447,5.463)}
\gppoint{gp mark 6}{(2.882,5.297)}
\gpcolor{color=gp lt color border}
\node[gp node right,font={\fontsize{10pt}{12pt}\selectfont}] at (2.240,4.989) {LSPI};
\gpcolor{rgb color={0.000,0.000,1.000}}
\gpsetlinetype{gp lt plot 5}
\draw[gp path] (2.424,4.989)--(3.340,4.989);
\draw[gp path] (1.320,1.330)--(1.935,1.335)--(2.748,1.420)--(3.362,1.520)--(3.977,1.702)%
  --(4.790,2.131)--(5.405,2.560)--(6.019,2.742)--(6.832,2.998)--(7.447,2.641);
\gppoint{gp mark 2}{(1.320,1.330)}
\gppoint{gp mark 2}{(1.935,1.335)}
\gppoint{gp mark 2}{(2.748,1.420)}
\gppoint{gp mark 2}{(3.362,1.520)}
\gppoint{gp mark 2}{(3.977,1.702)}
\gppoint{gp mark 2}{(4.790,2.131)}
\gppoint{gp mark 2}{(5.405,2.560)}
\gppoint{gp mark 2}{(6.019,2.742)}
\gppoint{gp mark 2}{(6.832,2.998)}
\gppoint{gp mark 2}{(7.447,2.641)}
\gppoint{gp mark 2}{(2.882,4.989)}
%% coordinates of the plot area
\gpdefrectangularnode{gp plot 1}{\pgfpoint{1.320cm}{0.985cm}}{\pgfpoint{7.447cm}{5.631cm}}
\end{tikzpicture}
%% gnuplot variables
    \label{fig:offline-pendulum}
  }
  \caption{\textbf{Off-line performance.} For $\nsamples= 10^3$, $\threshold = 10^{-2}$, $\ntrajectories=10^2$, $\nrollouts=2\cdot 10^3$, $\discount = 0.99$. The data are averaged over $10^2$ runs, with each run being evaluated with $10^3$ trajectories.  The shaded regions show $95\%$ bootstrap confidence intervals from $10^3$ bootstrap samples.} 
  \label{fig:offline}
\end{figure}

We compared the offline performance of LSPI and ABC-LSPI on the two domains. We first observe $\ntrain$ trajectories in the real environment drawn using a uniformly random policy. These trajectories are used by both ABC-LSPI and LSPI to estimate a policy. This policy is then evaluated over $10^3$ trajectories. The experiment was repeated for $10^2$ runs. Since LSPI requires a basis, in both cases we employed a uniform $4 \times 4$ grid of RBFs, as well as an additional unit basis for the value function estimation.

The results of the experiment are shown in Fig.~\ref{fig:offline}, where we plot the expected utility (with a discount factor $\gamma=0.99$) of the policy found as the number of trajectories increase. Both LSPI and ABC-LSPI  manage to find an improved policy with more data. However, the source of their improvement is different. In the case of LSPI, the additional data leads to better estimation of the value function. In ABC-LSPI, the additional data leads to a better sampled model. The value function is then estimated using a large number of rollouts in the sampled model. The CPU time taken by ABC ranges in 20 to 40s, versus 0.05 to 30s for pure LSPI, depending on the amount of training data. This is due to the additional overhead of sampling as well as the increased amount of rollouts used for ADP.

In general, the ABC approach quickly reaches a good performance, but then has little improvement. This effect is particularly prominent in the Mountain Car domain (Fig.~\ref{fig:offline-mountain-car}), where it is significantly worse asymptotically than LSPI. This can be attributed to the fact that even though more data is available, the number of samples drawn from the prior is not sufficient for a good model to be found. In fact, upon investigation we noticed that although most model parameters were reliably estimated, there was a difficulty in estimating the goal location from the given trajectories. This was probably the main reason why ABC didn't reach optimal performance in this case. However, it may be possible to improve upon this result with a more efficient sampling scheme, or a statistic that is closer to sufficiency than the simple utility-based statistic we used.

On the other hand, the performance is significantly better than LSPI in the pendulum environment (Fig.~\ref{fig:offline-pendulum}). There are two possible reason for this. Firstly,  ABC-LSPI not only uses more samples for the value function estimation, but also better distributed samples, as it estimates the value function by drawing trajectories starting from uniformly drawn states in the sampled environment. Secondly, and perhaps more importantly, that even for very differently parametrised pendulum problems the optimal policies on the pendulum domain are quite similar. Thus, even if ABC only samples a very approximate simulator, its optimal policy is going to be close to that of the real environment.

\section{Conclusion}
\label{sec:conclusion}
We presented an extension of ABC, a likelihood-free method for approximate Bayesian computation, to controlled dynamical systems. This method is particularly interesting for domains where it is difficult to specify an appropriate probabilistic model, and where computation is significantly cheaper than data collection. It is in principle generally applicable to any type of reinforcement learning problem, including continuous, partially observable and multi-agent domains. We also introduce a general theorem for the quality of the approximate ABC posterior distribution, which can be used for further analysis of ABC methods.

We then applied ABC inference to reinforcement learning. This involves using simulation both to estimate approximate posterior distributions and to find good policies. Thus, ABC-RL can be simultaneously seen as an extension of ABC inference to control problems and an extension of approximate dynamic programming methods to likelihood-free approximate Bayesian inference. The main advantage is when have no reasonable probabilistic model, but we do have access to a parametrised set of simulators, which contain good approximations to the real environment. This is frequently the case in complex control problems. However, we see that ABC-RL (specifically ABC-LSPI) is competitive with pure LSPI even in problems with low dimensionality where LSPI is expected to perform quite well. 

ABC-RL appears a viable approach, even with a very simple sampling scheme, and a utility-based statistic. In future work, we would like to investigate more elaborate ABC schemes such as Markov chain Monte Carlo, as well as statistics that are closer to sufficient, such as discounted feature expectations and conditional utilities. This would enable us to examine its performance in more complex problems where the practical advantages of ABC would be more evident. However, we believe that the results are extremely encouraging and that the ABC methodology has great potential in the field of reinforcement learning.

\appendix
\section{Collected proofs}
\begin{proof}[Proof of Remark~\ref{rem:policy-dependence}]
  Let $h = (\seq{x}{T+1}, \seq{a}{T}, \seq{r}{T})$. Using induction,
  \begin{align*}
    \Pmp(h)
    =
    \prod_{t=0}^T 
    \mdp_t(x_{t+1})
    \pol_t(a_{t})
    %% \Pr_{\mdp, \pol}(\seq{x}{T+1}, \seq{a}{T})
    %% &=
    %% \prod_{t=0}^T 
    %% \Pr_{\mdp}(x_{t+1} \mid \seq{x}{t}, \seq{a}{t})
    %% \Pr_{\pol}(a_{t} \mid \seq{x}{t}, \seq{a}{t-1})
    .
  \end{align*}
  Replacing in the posterior calculation \eqref{eq:posterior} we obtain:
  \begin{align}
    \bel(B \mid h, \pol) 
    &=
    \frac{\int_B \prod_{t=0}^T \mdp_t(x_{t+1}) \dd{\bel}(\mdp)}
    {\int_\MDPs \prod_{t=0}^T \mdp_t(x_{t+1}) \dd{\bel}(\mdp)}
    % \frac{\int_B \prod_{t=0}^T \Pr_{\mdp}(x_{t+1} \mid \seq{x}{t}, \seq{a}{t}) \dd{\bel}(\mdp)}
    % {\int_\MDPs \prod_{t=0}^T \Pr_{\mdp}(x_{t+1} \mid \seq{x}{t}, \seq{a}{t}) 
    \label{eq:posterior}
  \end{align}
  since the
  %% $\prod_{t=0}^T \Pr_{\pol}(a_{t} \mid \seq{x}{t}, \seq{a}{t-1})$ 
  $\prod_{t=0}^T \pol_t(a_{t})$
  terms can be taken out of the integrals and cancel out.
\end{proof}
\begin{proof}[Proof of Corollary~\ref{cor:statistic}]
  By definition, a sufficient statistic $f : \CH \to \CW$ has the following property:
  \begin{equation}
    \forall \mdp,\pol: ~ \Pmp(h) = \Pmp(h')
    \qquad
    \textrm{iff $f(h) = f(h')$}.
    \label{eq:sufficient-statistic}
  \end{equation}
  The probability of drawing a model in $B \subset \MDPs$ is:
  \begin{align}
    &\frac{\int_B \sum_{z \in \CH} \ind{f(z) = f(\hist)} \Pmp(z) \dd\bel(\mdp)}
    {\int_\MDPs \sum_{z \in \CH} \ind{f(z) = f(\hist)} \Pmp(z) \dd\bel(\mdp)}\nonumber
    \\
    =
    &\frac{\int_B \Pmp(\hist) \dd\bel(\mdp)}
    {\int_\MDPs \Pmp(\hist) \dd\bel(\mdp)}
    =
    \bel(B \mid \hist, \pol),
    \label{eq:posterior-2}
\end{align}
due to \eqref{eq:sufficient-statistic}.
\end{proof}

\begin{proof}[Proof of Theorem~\ref{the:posterior-divergence}]
For notational simplicity, we  introduce $\marg(\cdot) = \int_\MDPs \Pmp(\cdot) \dd\bel(\mdp)$ for the marginal prior measure on $\CH$, also omitting the dependency on $\pol$.
Then the ABC posterior $\bel_\epsilon(B\mid\hist)$ equals:
  \begin{align}
    &\frac{\int_B \sum_{z \in \CH} \ind{\|f(z) - f(\hist)\|<\epsilon} \Pmp(z) \dd\bel(\mdp)}
    {\int_\MDPs \sum_{z \in \CH} \ind{\|f(z) - f(\hist)\| < \epsilon} \Pmp(z) \dd\bel(\mdp)}\nonumber
    \\
    =
    &\frac{\int_B \Pmp(\Ae) \dd\bel(\mdp)}
    {\int_\MDPs \Pmp(\Ae) \dd\bel(\mdp)}
    =\frac{\int_B \Pmp(\Ae) \dd\bel(\mdp)}
    {\marg(\Ae)}.
    \label{eq:posterior-3}
\end{align}
From Definition~\ref{def:divergence}:
  \begin{align*}
    &\dist{\bel(\cdot \mid \hist)}{\bel_\epsilon(\cdot \mid \hist)} 
    = 
    \int_{B}
    \ln \frac{\dd{\bel}(\mdp \mid \hist)}{\dd{\bel_\epsilon}(\mdp \mid h)} \dd{\bel}(\mdp \mid \hist)
    \\
    &\overset{(a)}{=}
    \int_{\MDPs}
    \ln \left(
      \frac{\Pmp(\hist)}{\Pmp(\Ae)}
      \times
      \frac{\marg(\Ae)}{\marg(\hist)}
    \right)
    \dd{\bel}(\mdp \mid \hist)
    \\
    &= 
    \int_{\MDPs}
    \left(
    \ln 
    \frac{\Pmp(\hist)}{\Pmp(\Ae)}
    \dd{\bel}(\mdp \mid \hist)
    +
    \ln 
    \frac{\marg(\Ae)}{\marg(\hist)}
    \right)
    \dd{\bel}(\mdp \mid \hist)
    \\
    &\overset{(b)}{\leq}
    \int_{\MDPs}
    \left(
    \ln 
      \frac{\Pmp(\hist)}{\min_{z \in \Ae} \Pmp(z)}
      +
      \ln 
      \frac{\marg(\Ae)}{\marg(\hist)}
    \right)
    \dd{\bel}(\mdp \mid \hist)
    \\
    &\overset{(c)}{\leq}
    \int_{\MDPs}
    \left(
    \left|
      \ln 
      \frac{\Pmp(\hist)}{\min_{z \in \Ae} \Pmp(z)}
    \right|
    +
    \left|
    \ln 
    \frac{\marg(\Ae)}{\marg(\hist)}
    \right|
    \right)
    \dd{\bel}(\mdp \mid \hist)
    \\
    &
    \overset{(d)}{\leq}
    L\epsilon
    +
    \left|\ln \frac{\marg(\Ae)}{\marg(\hist)}\right|
    \overset{(e)}{\leq}
    %L\epsilon(1 + \ln |\Ae|).
    2L\epsilon + \ln|\Ae|.
  \end{align*}
  Equality (a) follows from  equations \eqref{eq:posterior-2} and \eqref{eq:posterior-3}. Inequality (b) follows from the fact that $\Pmp(\Ae) = \sum_{z \in \Ae} \Pmp(z) \geq  \min_{z \in \Ae} \Pmp(z)$, while (c) follows from $|x| \geq x$. For (d), first note that for any $z \in \Ae$, by the definition of $\Ae$, $|\ln[\Pmp(\hist)/\Pmp(z)]| \leq L \epsilon$, by Assumption~\ref{ass:lipschitz}, which can be taken out of the integral.
 The second $|\cdot|$ term in the integral is independent of $\mdp$ and so is also taken out. We can then bound the integral using $\int_\MDPs \bel(\mdp \mid \hist) = \bel(\MDPs \mid \hist) = 1$.
For (e), Assumption~\ref{ass:lipschitz} gives that
% $|\ln \marg(\Ae) / \marg(\hist)| \leq max_{z \in \Ae} |\ln |\Ae| \marg(z) / \marg(\hist)| \leq \ln |\Ae| + \max_{z \in \Ae} | \ln \marg(z) / \marg(\hist)| \leq \ln |\Ae| + L\epsilon$.
$\marg(z) = \int_\MDPs \Pmp(z) \dd\bel(\mdp) \leq \exp(L \epsilon) \marg(\hist)$ for any $z \in \Ae$ so, $\ln [\marg(\Ae)/\marg(\hist)] \leq L \epsilon + \ln |\Ae|$. Finally, as $\hist \in \Ae$,  $\marg(\Ae)\geq\marg(\hist)$ by additivity of measures, so the $|\cdot|$ can be removed, thus obtaining the final result.

\end{proof}

\bibliographystyle{icml2013}
\bibliography{../../bib/misc,../../bib/mine,../../bib/my_citations}

\end{document}